\documentclass{article}

\usepackage[preprint]{neurips_2026}

\usepackage[utf8]{inputenc}
\usepackage[T1]{fontenc}
\usepackage{hyperref}
\usepackage{url}
\usepackage{booktabs}
\usepackage{amsfonts}
\usepackage{amsmath}
\usepackage{amssymb}
\usepackage{amsthm}
\usepackage{mathtools}
\usepackage{microtype}
\usepackage{xcolor}
\usepackage{graphicx}
\usepackage{subcaption}
\usepackage{array}
\usepackage{booktabs}

\setcitestyle{authoryear,round}
\bibpunct{[}{]}{,}{a}{,}{,}

\newtheorem{theorem}{Theorem}
\newtheorem{assumption}[theorem]{Assumption}
\newtheorem{proposition}[theorem]{Proposition}
\newtheorem{lemma}[theorem]{Lemma}
\newtheorem{corollary}[theorem]{Corollary}

\newtheorem*{remark}{Remark}

\newcommand{\R}{\mathbb{R}}
\newcommand{\E}{\mathbb{E}}
\newcommand{\norm}[1]{\left\lVert #1 \right\rVert}
\newcommand{\inner}[2]{\left\langle #1,#2 \right\rangle}
\newcommand{\op}{\mathrm{op}}
\newcommand{\sign}{\operatorname{sign}}
\newcommand{\spanof}{\operatorname{span}}
\newcommand{\diag}{\operatorname{diag}}
\newcommand{\AUC}{\operatorname{AUC}}
\newcommand{\Train}{\operatorname{Train}}
\newcommand{\dd}{\mathrm{d}}
\newcommand{\adam}{\mathrm{adam}}
\newcommand{\wdec}{\mathrm{wd}}
\DeclarePairedDelimiter{\abs}{\lvert}{\rvert}

\title{Revocable Learned State via Process Sidecars}

\author{%
  John Sweeney\\
  Sideplane AI\\
  \texttt{john.sweeney@sideplane.ai}\\
}

\begin{document}

\maketitle

\begin{abstract}
Language models are often adapted in stages: a public skill phase, a private memory phase, and a later safety phase that learns to refuse outputs tied to the remembered entities. Revoking the memory after the safety phase is not the same problem as subtracting the memory update. The later safety optimizer has transported the memory direction. We introduce \emph{process sidecars}, a two-coefficient edit family
\[
  \hat{\theta}(\lambda, \gamma) = \theta_{AMS} - \lambda \Delta_M - \gamma\, \hat{R}_{S\leftarrow M},
  \qquad
  \hat{R}_{S\leftarrow M} = \hat{J}_{S,\varepsilon}(\Delta_M) - \Delta_M,
\]
where $\hat{J}_{S,\varepsilon}$ is a centered secant through the realized future AdamW safety-training process. The implementation uses $\varepsilon = 1$ at the natural memory-edit scale; it reuses $\theta_{AMS}$ as the positive endpoint and computes one additional safety trace at $\theta_A - \Delta_M$. We prove two things. First, the exact sidecar (using the true transported direction $R_{S\leftarrow M}$, not the secant estimate) at $(\lambda, \gamma) = (1, 1)$ recovers the counterfactual safety-only oracle $\theta_{AS}$ up to second order; the proof treats AdamW as an augmented-state map over parameters, first moments, and second moments. Second, this process information is necessary: whenever future safety training bends the memory direction, every scalar task-arithmetic edit leaves first-order counterfactual error, while the process-sidecar edit is second-order accurate. Across Qwen-2.5-0.5B-Instruct, Qwen-2.5-1.5B-Instruct, and Llama-3.2-1B-Instruct (twenty trials each), the validation-selected 2D edit improves held-out refusal closure over naive task arithmetic in 60 of 60 trials, and over the $\gamma = \lambda$ process-JVP subfamily (the diagonal slice of the cached 2D grid, 2 of 24 cells) in 60 of 60 paired trials. Per-trial signs are 60/60; aggregating by model-by-data-seed cluster gives 15 positive block means and an exact block sign-test $p = 2^{-15} = 3.05 \times 10^{-5}$. A locked-protocol confirmatory replication on never-seen seeds, with documented sample-size amendments, reproduces the effect across the three principal scales and a contemporary 8B transformer ($70/70$ trials; 20 positive model-by-data-seed block means, $p = 2^{-20} \approx 9.5 \times 10^{-7}$). Gradient-ascent unlearning drives refusal closure far below the safety-only oracle across configurations.
\end{abstract}

\section{Introduction}
\label{sec:intro}

A deployed model rarely has one training history. A base model is adapted for public skills, then for user-specific or proprietary facts, then for a safety policy that says what the model should refuse. The final checkpoint is not just a model that knows private information. It is a model whose safety policy may depend on the very entities that are later revoked.

This paper studies that revocation problem. We are given a final checkpoint
\[
  \theta_A \xrightarrow{\;M\;} \theta_{AM} \xrightarrow{\;S\;} \theta_{AMS},
\]
where $M$ is a memory phase and $S$ is a subsequent safety phase. The target is the counterfactual oracle
\[
  \theta_{AS} = \Train_S(\theta_A),
\]
the model that would have received the same safety training but never learned the memory. The deployable revocation service edits the shipped artifact $\theta_{AMS}$ in place, preserving its trajectory-specific characteristics (data ordering, optimizer state evolution, downstream pipeline integration) that a counterfactually retrained model would not share. The selector uses $\theta_A$, $\theta_{AM}$, $\theta_{AMS}$, validation data, and one additional safety-procedure trace; it does not access $\theta_{AS}$ or any test metric. Our contribution: scalar task arithmetic along $\Delta_M$ is provably first-order incomplete whenever the safety phase transports the memory direction, and the ideal process-JVP sidecar cancels this error to second order.

The standard weight-space answer is task arithmetic: subtract the memory delta $\Delta_M = \theta_{AM} - \theta_A$ from $\theta_{AMS}$ \citep{ilharco2023}. That answer is incomplete. The memory update is not stored statically after the memory phase. It is acted on by the later safety optimizer. If $T = \Train_S$ is the realized safety-training map, then
\[
  \theta_{AMS} = T(\theta_A + \Delta_M) = \theta_{AS} + DT(\theta_A)\Delta_M + O(\norm{\Delta_M}^2).
\]
To first order in $\norm{\Delta_M}$, the direction present in the final model is therefore $DT(\theta_A)\Delta_M$, not $\Delta_M$. Naive task arithmetic silently assumes $DT(\theta_A) = I$.

Process sidecars estimate this transported direction. We decompose
\[
  DT(\theta_A)\Delta_M = \Delta_M + R_{S\leftarrow M},
  \qquad
  R_{S\leftarrow M} = (DT(\theta_A) - I)\Delta_M,
\]
and search the two-dimensional family
\[
  \hat{\theta}(\lambda, \gamma) = \theta_{AMS} - \lambda \Delta_M - \gamma\, \hat{R}_{S\leftarrow M}.
\]
The family contains naive task arithmetic at $\gamma = 0$ and the process-JVP line at $\gamma = \lambda$. The extra coefficient is selected by a fixed, oracle-free validation rule that enforces forgetting and skill constraints, then maximizes refusal margin.

The empirical separation appears directly in Figure~\ref{fig:frontier}: across Qwen-2.5-0.5B-Instruct, Qwen-2.5-1.5B-Instruct, and Llama-3.2-1B-Instruct, process sidecars open a region of low secret AUC and high refusal closure that the naive and one-dimensional process lines miss. Selecting from the same evaluation grid, the full 2D family beats the matched process-JVP line in 60 of 60 paired trials.

\begin{figure}[t]
  \centering
  \includegraphics[width=\linewidth]{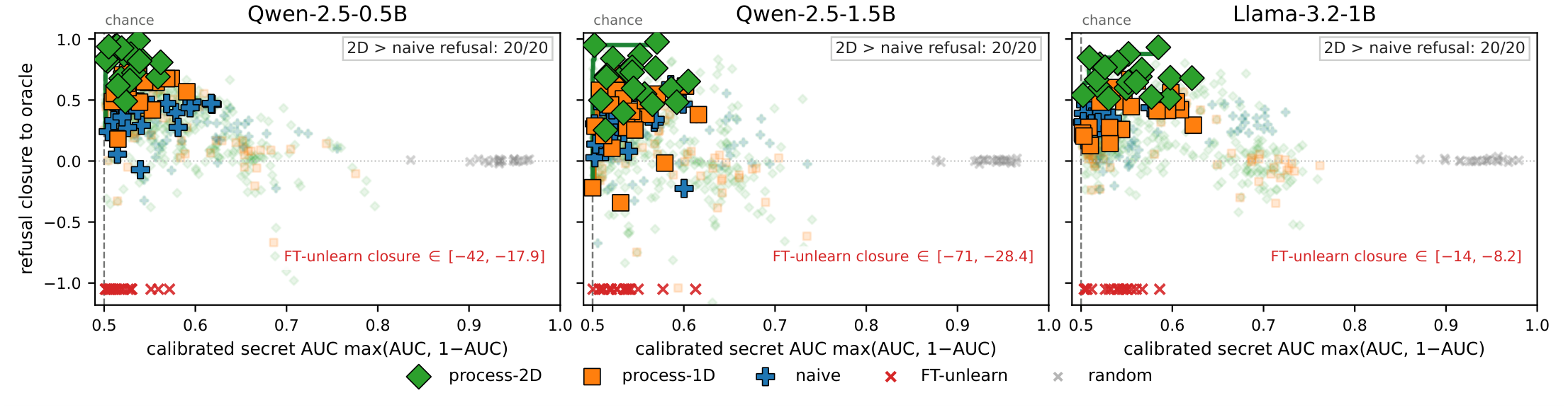}
  \caption{Process sidecars preserve refusal at matched secret-token distinguishability. Held-out test, 20 trials per scale. Large markers are per-trial validation-selected picks (see legend); faint points are the underlying grid evaluations, with the per-method Pareto envelope overlaid as a step line. The $x$-axis is the calibrated secret AUC, $\max(\AUC, 1-\AUC)$, where $0.5$ means no exploitable rank signal; $y$ is refusal closure to the oracle ($1$ matches the safety-only oracle, $0$ matches the unedited final checkpoint $\theta_{AMS}$). All three selectors leave only near-chance secret distinguishability ($\approx 0.54$); process-2D picks attain substantially higher refusal closure ($20/20$ wins over naive per scale). FT-unlearn is clipped at the bottom; in-panel annotations give the per-trial closure range across all FT-unlearn evaluations (six-config grid plus single-config seed sweep), wider than the per-config min/max in Table~\ref{tab:ftunlearn}.}
  \label{fig:frontier}
\end{figure}

The paper makes three claims. First, process sidecars are the right first-order object. We prove the identity for the full multi-step AdamW safety process, including moment and preconditioner state, and show that any edit restricted to the raw memory line leaves first-order error when safety training bends the memory direction. Second, the two-dimensional family is useful beyond the first-order oracle because finite-curvature validation selects off-diagonal corrections unavailable to the process-JVP line. Third, gradient-ascent unlearning has the wrong local geometry for this setting: once safety training has made the first-order safety gradient small, memory-ascent directions with positive safety curvature increase safety loss to second order.

\section{Setting and evaluation}
\label{sec:setting}

The checkpoint sequence is $\theta_A \to \theta_{AM} \to \theta_{AMS}$. The memory phase $M$ trains on a private fact set $D_M$. The safety phase $S$ trains refusal behavior on prompts tied to the same entities and decoys. The counterfactual oracle is $\theta_{AS} = \Train_S(\theta_A)$, using the same safety data, optimizer schedule, seed, and minibatch order as the observed $S$ phase. Conditioning on this realized trace makes $T = \Train_S$ a deterministic map.

All experiments use trainable-coordinate vectors. For LoRA runs, $\theta$ denotes the adapter-coordinate vector; for full-parameter baselines, it denotes the full trainable vector. The main runs use LoRA rank 8 on attention and MLP modules \citep{hu2022lora}. The memory data are high-entropy canaries of the form \texttt{AAA-1234-AAA-1234} paired with project identifiers. The structured-synthetic canary uses profile-like records rather than raw token strings.

We evaluate three quantities. \emph{Forgetting} is the recoverable-signal of the true secret token among 31 decoys. We report calibrated secret AUC,
\[
  \AUC^{\mathrm{cal}}_{\mathrm{secret}}(\theta) = \max\!\big(\AUC_{\mathrm{secret}}(\theta),\, 1 - \AUC_{\mathrm{secret}}(\theta)\big),
\]
which equals $0.5$ when the secret is unrankable above decoys and $1.0$ under perfect memorization. Calibrated AUC is the right secrecy metric because raw $\AUC < 0.5$ is anti-ranking, not better forgetting: an adaptive evaluator can flip the score and recover the same signal. \emph{Refusal preservation} is a signed log-probability margin $m(\theta) = \log p_\theta(\text{refuse} \mid \text{prompt}) - \log p_\theta(\text{comply} \mid \text{prompt})$ between refusal and comply continuations on entity-tied prompts, reported as closure to the oracle:
\[
  \mathrm{closure}_{\mathrm{refusal}}(\theta) = \frac{m(\theta) - m(\theta_{AMS})}{m(\theta_{AS}) - m(\theta_{AMS})},
\]
so closure $1$ means the edit recovers the oracle margin and closure $0$ means it remains at $\theta_{AMS}$. Negative closure means the edit's refusal margin is below $\theta_{AMS}$. \emph{Skill retention} is a one-sided $\theta_{AMS}$-only score: $1$ if the edit's public-skill validation NLL stays within tolerance of $\theta_{AMS}$, decaying linearly above it. It uses no oracle quantity. In Table~\ref{tab:main} we report the calibrated-AUC gap $\AUC^{\mathrm{cal}}_{\mathrm{secret}}(\theta_{\mathrm{2D}}) - \AUC^{\mathrm{cal}}_{\mathrm{secret}}(\theta_{\mathrm{naive}})$; values near zero mean the two methods leave the secret comparably distinguishable. The deployable selector never uses $\theta_{AS}$ or test metrics. It filters candidates by validation secret AUC $\le 0.60$, skill retention $\ge 0.90$, and refusal preference rate $\ge 0.99$; among feasible candidates it maximizes
\[
  \log\!\big(1 + \max\{0,\, m_{\mathrm{edit}} - m_{AMS}\}\big),
\]
where $m$ is the validation refusal margin, and then breaks ties by minimum edit norm.

\section{Method}
\label{sec:method}

\paragraph{Estimating the transported memory direction.}
Let $u = \Delta_M$. The process-JVP direction is $J_S(u) = DT(\theta_A)u$, a Jacobian-vector product in the forward-sensitivity sense \citep{pearlmutter1994,maclaurin2015,lorraine2020}. Computing the exact Jacobian of a full AdamW fine-tuning run is unnecessary. We estimate the product by a centered secant through the same future safety-training procedure:
\begin{equation}
  \hat{J}_{S,\varepsilon}(u) = \frac{T(\theta_A + \varepsilon u) - T(\theta_A - \varepsilon u)}{2\varepsilon}.
  \label{eq:secant}
\end{equation}
The implementation uses $\varepsilon = 1$ because $\norm{\Delta_M}$ is the natural scale of the memory edit. The positive endpoint is the shipped artifact: $T(\theta_A + \Delta_M) = T(\theta_{AM}) = \theta_{AMS}$, so only the negative endpoint $T(\theta_A - \Delta_M)$ is computed during revocation. The local theorem in Section~\ref{sec:theory} gives the limiting justification; the experiments evaluate the finite-scale estimator directly.

\paragraph{The edit family.}
The residual sidecar is
\[
  \hat{R}_{S\leftarrow M} = \hat{J}_{S,\varepsilon}(\Delta_M) - \Delta_M.
\]
We search
\begin{equation}
  \hat{\theta}(\lambda, \gamma) = \theta_{AMS} - \lambda \Delta_M - \gamma\, \hat{R}_{S\leftarrow M}.
  \label{eq:family}
\end{equation}
Three subfamilies matter. Naive task arithmetic is $\gamma = 0$. The process-JVP line is $\gamma = \lambda$, since then the edit subtracts $\lambda J_S(\Delta_M)$. The full 2D family allows $\gamma$ to move independently.

\begin{figure}[t]
  \centering
  \includegraphics[width=0.72\linewidth]{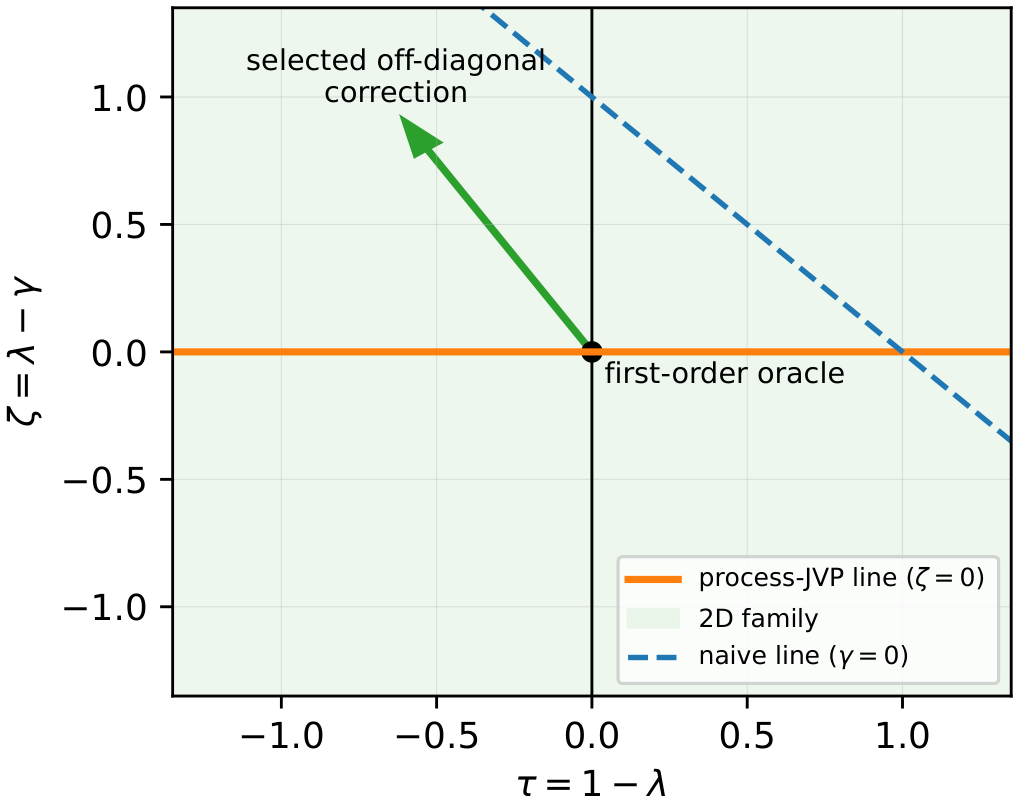}
  \caption{Natural coordinates for process sidecars. The process-JVP line is $\zeta = 0$. The full 2D family can choose an off-diagonal curvature correction, while naive task arithmetic lies on $\gamma = 0$ rather than on the process line.}
  \label{fig:coords}
\end{figure}

It is cleaner to write the finite-curvature degrees of freedom as
\[
  \tau = 1 - \lambda,
  \qquad
  \zeta = \lambda - \gamma.
\]
In these coordinates the process-JVP line is exactly $\zeta = 0$ (Figure~\ref{fig:coords}). The raw sign of $\gamma$ has no invariant meaning; the off-diagonal coefficient is $\zeta$. This matters empirically: at Qwen-2.5-1.5B-Instruct the selected $\gamma$ is bimodal, with 14 trials selecting $+2.0$ and 6 selecting $-0.5$. The theory predicts trial-local curvature alignment, not a scale law for $\gamma$.

\paragraph{Baselines.}
We compare against random norm-matched edits, naive task arithmetic, the process-JVP line, gradient-ascent FT-unlearning, Negative Preference Optimization \citep{zhang2024npo}, and the WMDP/RMU base-parameter code path \citep{li2024wmdp}. RMU updates selected base-model parameters rather than LoRA adapter coordinates; behavioral metrics are comparable, but weight-space oracle closure is not reported for that row.

\section{Theory}
\label{sec:theory}

Future safety training transports the memory direction. This section formalizes that transport, shows how the sidecar family recovers the counterfactual safety-only oracle, and explains why the two-dimensional family differs from the process-JVP line at finite scale.

\begin{assumption}[Regular fixed safety trace]
\label{ass:trace}
The realized $K$-step AdamW safety-training trace is fixed: minibatches, dropout masks, optimizer hyperparameters, schedule, and seed are conditioned on. In a neighborhood of $\theta_A$, the losses are $C^3$, clipping decisions do not cross their nonsmooth boundary, and AdamW bias-corrected second moments on active coordinates stay bounded away from zero.
\end{assumption}

\begin{theorem}[Process-sidecar identity]
\label{thm:identity}
Under Assumption~\ref{ass:trace}, let $T = \Train_S$ be the realized safety-training map and $u = \Delta_M$. Then
\[
  \theta_{AMS} = T(\theta_A + u) = \theta_{AS} + DT(\theta_A)u + O(\norm{u}^2).
\]
Writing $J_S(u) = DT(\theta_A)u$ and $R_{S\leftarrow M} = J_S(u) - u$ gives the decomposition and the exact sidecar:
\[
  \theta_{AMS} = \theta_{AS} + u + R_{S\leftarrow M} + O(\norm{u}^2),
  \qquad
  \theta_{AMS} - u - R_{S\leftarrow M} = \theta_{AS} + O(\norm{u}^2).
\]
The empirical estimator $\hat{\theta}(1, 1)$ replaces $R_{S\leftarrow M}$ with the centered-secant $\hat{R}_{S\leftarrow M}$, incurring an additional $O(\varepsilon^2 \norm{u}^3)$ secant bias.
\end{theorem}

The theorem differentiates the whole training procedure. For AdamW, which combines Adam-style moment adaptation with decoupled weight decay \citep{kingma2015,loshchilov2019}, $DT(\theta_A)$ is not a single Hessian term. It is the Jacobian of an augmented map over
\[
  z_t = (\theta_t, m_t, v_t).
\]
The appendix gives the exact tangent recursion for $\dot{\theta}_t, \dot{m}_t, \dot{v}_t$, including bias correction, second-moment sensitivity, decoupled weight decay, and differentiable clipping branches. This is the process-level quantity estimated by Equation~\eqref{eq:secant}.

\begin{theorem}[First-order necessity of the process sidecar]
\label{thm:necessity}
Assume additionally that $\norm{D^2 T(\theta)}_\op \le L_2$ in the local neighborhood. Let $p = DT(\theta_A)u$, $r = p - u$, and let $P_u$ be Euclidean projection onto $\spanof\{u\}$. For the raw task-arithmetic line
\[
  \hat{\theta}_\lambda = T(\theta_A + u) - \lambda u,
\]
we have
\[
  \inf_{\lambda \in \R} \norm{\hat{\theta}_\lambda - T(\theta_A)} \ge \norm{(I - P_u)r} - \frac{L_2}{2}\norm{u}^2.
\]
By contrast, the process-sidecar edit
\[
  \hat{\theta}_{\mathrm{sc}} = T(\theta_A + u) - p = T(\theta_A + u) - u - r
\]
satisfies
\[
  \norm{\hat{\theta}_{\mathrm{sc}} - T(\theta_A)} \le \frac{L_2}{2}\norm{u}^2.
\]
Thus, whenever safety training bends the memory direction off the raw memory line, i.e.\ $\norm{(I - P_u)r} \ge \rho\norm{u}$, every scalar task-arithmetic edit leaves first-order counterfactual error while the process sidecar is second-order accurate.
\end{theorem}

\begin{proof}
Taylor's theorem gives $T(\theta_A + u) = T(\theta_A) + p + q$ with $\norm{q} \le L_2 \norm{u}^2 / 2$. For any $\lambda$, $\hat{\theta}_\lambda - T(\theta_A) = (1 - \lambda)u + r + q$. Projecting orthogonally to $\spanof\{u\}$ removes the scalar task-arithmetic term, so
\[
  \norm{(1 - \lambda)u + r + q} \ge \norm{(I - P_u)(r + q)} \ge \norm{(I - P_u)r} - \norm{q}.
\]
Taking the infimum over $\lambda$ gives the lower bound. The sidecar edit leaves residual $q$, giving the upper bound.
\end{proof}

Theorem~\ref{thm:necessity} is the sharp version of the naive-task-arithmetic critique. A better scalar coefficient on $\Delta_M$ cannot remove the component of the transported memory direction orthogonal to $\Delta_M$; the sidecar supplies exactly that missing component.

The 2D family and the process-JVP line both contain the first-order oracle. The role of the second coefficient appears at the next order. If $T \in C^3$, write
\[
  \theta_{AMS} = \theta_{AS} + p + q + O(\norm{u}^3),
  \qquad
  p = J_S(u),
  \qquad
  q = \tfrac{1}{2} D^2 T(\theta_A)[u, u].
\]
Since $r = p - u$, the residual after a 2D edit is
\begin{equation}
  \hat{\theta}(\lambda, \gamma) - \theta_{AS} = q + \tau p + \zeta r + O(\norm{u}^3),
  \qquad
  \tau = 1 - \lambda,
  \qquad
  \zeta = \lambda - \gamma.
  \label{eq:residual2d}
\end{equation}
The process-JVP line is the subfamily $\zeta = 0$.

\begin{theorem}[Local quadratic frontier]
\label{thm:frontier}
Let a validation scalarization near $\theta_{AS}$ have expansion
\[
  \ell(\theta_{AS} + e) = \ell_0 + g^\top e + \tfrac{1}{2} e^\top H e + O(\norm{e}^3),
\]
with $H$ positive definite on $\spanof\{p, r\}$. Let $W_1 = \spanof\{p\}$ and $W_2 = \spanof\{p, r\}$. The optimized local quadratic objective over $W_2$ is weakly no worse than over $W_1$. Let
\[
  r_\perp = r - P^H_{\spanof\{p\}} r.
\]
If $r_\perp \ne 0$, the improvement is strict exactly when
\[
  q^\top H r_\perp + g^\top r_\perp \ne 0,
\]
and the corresponding off-diagonal coordinate after profiling over $\tau$ is
\[
  \zeta^* = -\frac{q^\top H r_\perp + g^\top r_\perp}{\norm{r_\perp}^2_H}.
\]
\end{theorem}

\begin{remark}
Without the positive-definite or second-order-sufficient condition, the guarantee reduces to trust-region inclusion, $W_1 \cap B_\rho \subseteq W_2 \cap B_\rho$, plus empirical selection over the declared grid.
\end{remark}

Theorem~\ref{thm:frontier} locates the 2D advantage at second order rather than first order. Hessian anisotropy alone is insufficient; strict improvement requires a Hessian-weighted off-diagonal projection of the second-order/KKT correction onto a sidecar direction unreachable by the process-JVP line. Appendix Lemma~\ref{lem:sign} gives a sign criterion for indefinite projected Hessians when the objective is profiled over $\tau$ and the linear term dominates curvature across the searched $\zeta$ interval. Accordingly, the HVP sign computation is interpretive; the empirical certificate for the 2D family is the held-out selector comparison in Section~\ref{sec:offdiag}.

\begin{proposition}[Local drift of gradient-ascent unlearning]
\label{prop:drift}
For an FT-unlearning step $\theta_+ = \theta + \alpha d$, where $d$ is the memory-ascent update direction,
\[
  L_S(\theta_+) - L_S(\theta) = \alpha \inner{\nabla L_S(\theta)}{d} + \frac{\alpha^2}{2} d^\top \nabla^2 L_S(\theta) d + O(\alpha^3 \norm{d}^3).
\]
If safety training makes the first-order term small and $d$ has positive safety curvature, the step increases safety loss to second order.
\end{proposition}

This proposition identifies a local failure mode: memory ascent can move directly into high-curvature safety directions. It explains the direction of the FT-unlearn degradation in Section~\ref{sec:ftunlearn}; the observed magnitudes are empirical.

\section{Experiments}
\label{sec:experiments}

\subsection{Setup}
\label{sec:setup}

We evaluate on Qwen-2.5-0.5B-Instruct, Qwen-2.5-1.5B-Instruct \citep{qwen2024report,qwen2024_05b,qwen2024_15b}, and Llama-3.2-1B-Instruct \citep{metaai2024,grattafiori2024}. Each principal model has 20 independent trials. We also report a Qwen3.5-9B scale probe: the public Qwen3.5-9B Hugging Face checkpoint is a vision-language model (\texttt{Qwen3\_5ForConditionalGeneration} architecture) with a vision encoder \citep{qwen2026_9b_card,qwen2026_9b_config}. Our experiments use only the text-only path: memory canaries, validation prompts, and held-out refusal probes are all text; the vision encoder is unused. The language decoder is itself a hybrid: it interleaves Gated-DeltaNet and gated-attention layers, with every fourth layer using full attention. By ``hybrid'' we mean this attention layout, not a multimodal ensemble. The 9B scale probe is reported at $n = 11$ rather than $n = 20$ because of its larger per-trial compute cost. The same fresh-confirmatory protocol is also reported on a contemporary 8B pure-transformer (Qwen3-8B) \citep{qwen2025report,qwen2025_8b_card} to disentangle architectural confounds from raw scale. Each trial uses 64 memory canaries, held-out secret probes with 31 decoys per fact, 128 general refusal prompts, and entity-specific refusal probes. The 2D family is evaluated on a $4 \times 6$ cached grid ($\lambda \in \{0.75, 1.0, 1.25, 1.5\}$, $\gamma \in \{-0.5, 0.0, 0.5, 1.0, 1.5, 2.0\}$); naive and process-1D subfamilies are selected by the same validation rule. Confidence intervals are percentile bootstrap intervals over independent trials; paired method comparisons also report an exact one-sided sign test.

\subsection{Validation-selected process sidecars preserve refusal after forgetting}
\label{sec:mainresults}

Table~\ref{tab:main} reports the deployable selector. The selected process-sidecar edit beats naive task arithmetic on refusal closure in every trial across the three model settings. The two methods leave the secret comparably distinguishable on calibrated AUC (mean gap within $\pm 0.034$ at every scale), so the refusal advantage is not produced by sacrificing forgetting.

\begin{table}[t]
  \caption{Validation-selected process sidecars versus naive task arithmetic. Refusal closure is closure to the safety-only oracle on held-out test probes. $\Delta\AUC^{\mathrm{cal}}$ is the calibrated-AUC gap (process-2D minus naive); values near zero mean both methods leave the secret comparably unrankable. The combined sign test over the three full rows is $60/60$; the 9B row is a scale-probe run.}
  \label{tab:main}
  \centering
  \begin{tabular}{lccccc}
    \toprule
    Model & $n$ & $\Delta$refusal & 95\% CI & sign & $\Delta\AUC^{\mathrm{cal}}$ \\
    \midrule
    Qwen-2.5-0.5B-Instruct & 20 & $+0.438$ & $[+0.383, +0.490]$ & $20/20$ & $-0.024$ \\
    Qwen-2.5-1.5B-Instruct & 20 & $+0.347$ & $[+0.273, +0.435]$ & $20/20$ & $+0.009$ \\
    Llama-3.2-1B-Instruct  & 20 & $+0.334$ & $[+0.289, +0.379]$ & $20/20$ & $+0.034$ \\
    Qwen3.5-9B scale probe & 11 & $+0.283$ & $[+0.174, +0.433]$ & $11/11$ & $+0.005$ \\
    \bottomrule
  \end{tabular}
\end{table}

The frontier view gives the same message. At Qwen-2.5-1.5B-Instruct, process-2D dominates the entire naive frontier in $18/20$ trials, with 95\% mean naive-frontier coverage. The corresponding values are $12/20$ trials and 83\% coverage at Qwen-2.5-0.5B-Instruct, and $15/20$ trials and 87\% coverage at Llama-3.2-1B-Instruct.

\subsection{The 2D gain is the value of off-diagonal cells}
\label{sec:offdiag}

The 2D grid has more candidates than the 1D subfamilies by construction (24 vs 2 on the $\gamma = \lambda$ diagonal, vs 4 on $\gamma = 0$). We separate the contribution of the additional axis from a search-space artifact in two stages: first the family-restriction comparison, then a candidate-count-controlled mean-best-of-$K$ analysis.

\paragraph{Family restriction.}
Restricting the same cached evaluation to a subfamily of the 2D grid lets us ask: does the off-diagonal axis ($\gamma \ne \lambda$, $\gamma \ne 0$) help, or is the diagonal already enough? The full 2D family beats every subfamily under the same selector, with no new evaluation:

\begin{table}[t]
  \caption{Subfamily comparison. The full 2D family beats the matched-1D process-JVP diagonal slice ($\gamma = \lambda$, 2 cells) and the matched-naive slice ($\gamma = 0$, 4 cells), evaluated on the same cached grid under the same selector.}
  \label{tab:subfamily}
  \centering
  \begin{tabular}{lcccc}
    \toprule
    Comparison & $n$ & mean $\Delta$refusal & 95\% CI & sign test \\
    \midrule
    2D vs matched-1D ($\gamma = \lambda$, 2 cells) & 60 & $+0.316$ & $[+0.260, +0.381]$ & $60/60$ \\
    2D vs matched-naive ($\gamma = 0$, 4 cells)    & 60 & $+0.373$ & $[+0.335, +0.412]$ & $60/60$ \\
    \bottomrule
  \end{tabular}
\end{table}

\paragraph{Equal-density diagonal control.}
The above subfamily comparison restricts the JVP line to 2 cells of the cached grid, so candidate count is not equalized. We therefore evaluate a locked-protocol dense process-JVP diagonal: 24 evenly spaced $\lambda \in [0.75, 1.5]$ with $\gamma = \lambda$, using each trial's cached states ($\theta_A, \theta_{AM}, \theta_{AMS}, \hat{R}_{S\leftarrow M}$) and the identical selector. The dense-diagonal family now has the same number of validation candidates as the 2D family along the $\zeta = 0$ direction; the comparison rules out a count effect parallel to $\zeta = 0$ but does not control for arbitrary alternative directions in $\R^d$.

At each of the three principal scales, the full 2D family beats the dense process-JVP diagonal in $20/20$ trials: mean refusal-closure gap $+0.305$ at Qwen-2.5-1.5B-Instruct, $+0.202$ at Qwen-2.5-0.5B-Instruct, and $+0.243$ at Llama-3.2-1B-Instruct. Each scale's exact one-sided sign test (and Wilcoxon signed-rank) gives $p = 9.5 \times 10^{-7}$. The off-diagonal cells $\gamma \ne \lambda$ therefore carry the 2D advantage even when the JVP-line baseline has matched candidate count.

\subsection{Seed-disjoint matched-family evidence}
\label{sec:seeddisjoint}

The basic paired sign test treats the 60 trial signs as exchangeable Bernoulli$(1/2)$ under $H_0$; all 60 paired differences are positive for 2D versus matched-1D. To avoid making the 60-trial independence assumption load-bearing, we also aggregate the trials into model-by-data-seed blocks. Each principal model has five data-seed clusters with four trials per cluster, giving 15 blocks total; within a block we allow arbitrary dependence and only require that the block-mean signs are independent and Bernoulli$(1/2)$-symmetric across blocks under $H_0$ (different models are independent runs; different data-seed clusters use disjoint canary subsets and validation prompts). All 15 block means are positive for 2D versus matched-1D, giving the exact one-sided block sign-test $p = 2^{-15} = 3.05 \times 10^{-5}$. The same $15/15$ block result holds for 2D versus matched-naive. This is the conservative statistical certificate we use for the matched-family claim.

If one treats individual seed-disjoint trials as exchangeable, each per-scale exact one-sided sign test is $20/20$ and gives $p = 2^{-20} \approx 9.54 \times 10^{-7}$, and Fisher's method across the three independent model strata gives $p \approx 7.9 \times 10^{-16}$. We report these as secondary summaries under the per-trial-independence assumption.

The selected off-diagonal sign is trial-local, not a scale law. Across 20 Qwen-2.5-1.5B-Instruct trials and 11 Qwen3.5-9B trials, the selector is bimodal in the natural coordinate $\zeta = \lambda - \gamma$: $14/20$ versus $6/20$ at the 1.5B scale and $6/11$ versus $5/11$ at the 9B scale, with both signs co-occurring under the same procedure on the same model. We use finite-curvature probes of the underlying Hessian for interpretation rather than as selector evidence: the projected Hessian is positive definite in $3/8$ cached 1.5B probes, and the matched-family certificate above does not rely on the profiled-$\tau$ linear-dominance check.

\subsection{Locked-protocol fresh-confirmatory replication}
\label{sec:fresh}

The matched-family analysis above uses cached evaluation grids. To rule out selection-iteration effects, we re-ran a locked protocol on never-seen data seeds, with documented sample-size amendments before the corresponding extension compute: four trials per seed-cluster at the principal scales ($n_{\mathrm{fresh}} = 20$) and a smaller average cluster size for the 8B contemporary control ($n_{\mathrm{fresh}} = 10$, across the same five clusters). Across the four scales below, the validation-selected process-2D edit beats naive task arithmetic in every fresh trial:

\begin{table}[t]
  \caption{Fresh-confirmatory replication. New trials per scale under the locked-and-amended protocol, on seeds disjoint from the cached evaluation grid. The three principal scales reach the amended $n = 20$. Qwen3-8B \citep{qwen2025report,qwen2025_8b_card} is a contemporary pure-transformer added to disentangle the Gated-DeltaNet 9B scale-probe result from raw scale; the 8B run is at $n = 10$.}
  \label{tab:fresh}
  \centering
  \begin{tabular}{lcccc}
    \toprule
    Model & $n_{\mathrm{fresh}}$ & mean $\Delta$refusal & 95\% CI & sign \\
    \midrule
    Qwen-2.5-0.5B-Instruct & 20 & $+0.402$ & $[+0.351, +0.452]$ & $20/20$ \\
    Qwen-2.5-1.5B-Instruct & 20 & $+0.323$ & $[+0.273, +0.373]$ & $20/20$ \\
    Llama-3.2-1B-Instruct  & 20 & $+0.235$ & $[+0.202, +0.270]$ & $20/20$ \\
    Qwen3-8B               & 10 & $+0.354$ & $[+0.297, +0.413]$ & $10/10$ \\
    \bottomrule
  \end{tabular}
\end{table}

All 70 fresh-confirmatory trials under this locked-and-amended protocol are positive. Aggregating into model-by-data-seed clusters gives 20 positive block means out of 20 blocks (5 clusters per scale, 4 scales), with the exact one-sided block sign-test $p = 2^{-20} \approx 9.5 \times 10^{-7}$. Fisher's method across the four model runs gives $p = 1.7 \times 10^{-17}$ as a secondary summary under the stronger per-trial-independence assumption. The fresh per-model effect overlaps the cached 95\% bootstrap CI at Qwen-2.5-0.5B-Instruct and Qwen-2.5-1.5B-Instruct; at Llama-3.2-1B-Instruct the fresh mean ($+0.235$) sits below the cached lower bound ($+0.288$), so the cross-evaluation transfer at that scale is qualitatively replicated but quantitatively smaller. Bimodality of the selected $\gamma$ reproduces: at Qwen-2.5-1.5B-Instruct, $13/20$ fresh trials select $\gamma = -0.5$ and $7/20$ select $\gamma = +2.0$ (cached at the same scale: $14/20$ versus $6/20$, with $\gamma = +2.0$ dominant). The dominant raw sign therefore differs between cached and fresh at this scale, consistent with the trial-local interpretation of $\gamma$ rather than a scale law. At Qwen3-8B, $7/10$ fresh trials select $\gamma = -0.5$ and $3/10$ select $\gamma = +2.0$. The fresh trials replicate the 2D-vs-naive direction at every scale, replicate the trial-local bimodality, and extend the result to a contemporary 8B pure-transformer.

\subsection{Gradient-ascent unlearning collapses refusal preservation}
\label{sec:ftunlearn}

FT-unlearning succeeds at making the secret less extractable, but it destroys the refusal behavior that safety training installed. This is the failure mode predicted by Proposition~\ref{prop:drift}: memory ascent is not constrained to stay outside safety-curvature directions.

The FT-unlearn results are reported across two configurations. The six-config FT-unlearn grid sweeps six predeclared combinations of the loss weights $(\alpha_{\mathrm{forget}}, \alpha_{\mathrm{skill}}, \alpha_{\mathrm{safety}})$ (listed in Appendix~\ref{app:baselines}) at fixed learning rate $10^{-4}$ and one epoch. The single-config baseline sweep runs the default $\alpha = 1.0$, $\alpha_{\mathrm{skill}} = 0.3$, $\alpha_{\mathrm{safety}} = 0.5$ across five seeds.

\begin{table}[t]
  \caption{FT-unlearn refusal-margin degradation. Values are refusal closure to the safety-only oracle (min and max across configs/trials shown as a range). Negative values mean the edit is farther from the oracle than the observed final checkpoint $\theta_{AMS}$ along the refusal-margin axis. The six-config FT-unlearn grid varies $(\alpha_{\mathrm{forget}}, \alpha_{\mathrm{skill}}, \alpha_{\mathrm{safety}})$ as in Appendix~\ref{app:baselines}; the 9B row is a single-trial scale probe at the same protocol as the 1.5B baseline sweep.}
  \label{tab:ftunlearn}
  \centering
  \begin{tabular}{lcc}
    \toprule
    Setting & runs & refusal closure \\
    \midrule
    Qwen-2.5-0.5B-Instruct FT-unlearn grid & 6 configs & $-15$ to $-33$ \\
    Qwen-2.5-1.5B-Instruct FT-unlearn grid & 6 configs & $-33$ to $-55$ \\
    Qwen-2.5-1.5B-Instruct baseline sweep  & 5 trials  & mean $-43.4$, median $-47.1$ \\
    Qwen3.5-9B scale probe                 & 1 trial   & $-38.7$ \\
    Structured-synthetic canary, Qwen-2.5-1.5B-Instruct & 5 trials & mean $-146.8$ \\
    \bottomrule
  \end{tabular}
\end{table}

NPO is less damaging than FT-unlearn but still gives negative refusal closure in the 1.5B baseline sweep, with mean refusal closure $-1.33$. RMU changes the metrics little in the tested base-parameter configuration: secret AUC remains high rather than reaching chance, and refusal closure is near zero. Full hyperparameters for both baselines are in Appendix~\ref{app:baselines}.

\paragraph{Model-merging baselines (TIES, DARE).}
The necessity theorem predicts that any edit confined to a sparsified copy of $\Delta_M$ inherits the first-order incompleteness of naive task arithmetic. We test this on the same five 1.5B trials by sparsifying $\Delta_M$ via TIES \citep{yadav2023} (per-tensor magnitude trim from mergekit \citep{goddard2024}) and DARE \citep{yu2024} (Bernoulli drop with $1/d$ rescale), then applying the selector. With densities $\{0.2, 0.5, 0.8\}$ and $\lambda \in \{0.75, 1.0, 1.25, 1.5\}$, process-2D beats $\mathrm{TIES}_{\mathrm{best}}$ in $5/5$ trials with mean refusal-closure gap $+0.354$, and beats $\mathrm{DARE}_{\mathrm{best}}$ in $5/5$ with mean gap $+0.469$. TIES yields a small implicit-regularization gain over naive ($+0.110$ mean, $5/5$); DARE is statistically indistinguishable from naive ($-0.005$ mean, $4/5$). The theorem prediction therefore holds empirically: sparsifying the raw memory delta does not recover the transported direction.

\section{Discussion and limitations}
\label{sec:discussion}

The theoretical claim is first-order and process-level: it applies to a fixed smooth safety-training trace and explains why raw memory subtraction is incomplete. The finite $\varepsilon = 1$ implementation is justified locally by the derivative theorem; the experiments are the evidence for finite-scale behavior.

The off-diagonal coefficient is not a scale law. The Qwen-2.5-1.5B-Instruct bimodality shows both signs within one scale. The invariant quantity is $\zeta = \lambda - \gamma$, determined by local validation geometry rather than by raw model size. The HVP cache supports this interpretation: only $3/8$ cached probes satisfy the positive-definite hypothesis needed for the closed-form local minimizer.

The frontier and sign-criterion theorems have hypothesis conditions; the held-out selector behavior is the empirical certificate. Proposition~\ref{prop:drift} predicts FT-unlearn's direction, not its magnitude. Multi-source revocation, repeated safety, adversarial extraction, broader RMU tuning, and higher-rank adapters remain open.

\section{Conclusion}
\label{sec:conclusion}

Process sidecars are the geometrically correct edit family for revoking learned memory under later safety training: first-order necessity against scalar task arithmetic and second-order accuracy of the ideal sidecar carry through empirically across four model scales and two architectures.

\bibliography{references}

\clearpage
\appendix

\section{Related work}
\label{app:related}

\paragraph{Weight-space editing and merging.}
Task arithmetic and model merging \citep{ilharco2023,wortsman2022,matena2022,yadav2023,yu2024,goddard2024} edit fine-tuned checkpoints by manipulating weight-space differences. Process sidecars keep this deploy-time editing form, but replace a static delta with an estimate of how the later safety-training process transports the memory direction.

\paragraph{Machine unlearning.}
Classical machine unlearning frames deletion against a retrained-without-the-data counterfactual, through exact retraining systems, certified removal, or approximate selective forgetting \citep{bourtoule2021,guo2020,golatkar2020}. LLM unlearning adapts this goal to memorized text, private facts, and hazardous capabilities, with gradient-ascent unlearning, TOFU, content unlearning benchmarks such as MUSE, WMDP/RMU, and preference-based methods such as NPO \citep{jang2023,eldan2023,maini2024,shi2025,li2024wmdp,zhang2024npo}. Recent robustness work shows that approximate LLM unlearning can suppress outputs without eliminating recoverable knowledge, especially under relearning or adversarial evaluation \citep{hu2025,lucki2025}. Our setting is different: the deletion target has been followed by an entity-tied safety phase, so the failure mode is not only retain-distribution damage but erasure of the refusal behavior protecting the memorized entities.

\paragraph{Safety post-training.}
Instruction and safety post-training methods such as RLHF, Constitutional AI/RLAIF, and DPO tune models toward preferred or harmless behavior after pretraining \citep{ouyang2022,bai2022,rafailov2023}. We model the safety phase as a realized deterministic training operator and ask how to revoke a prior memory without discarding the later refusal policy learned around that memory.

\paragraph{Sensitivity through training.}
Influence functions and TRAK trace predictions back to training data with Hessian-vector or projected sensitivity approximations \citep{koh2017,park2023}. Meta-learning and hyperparameter-optimization work differentiates through entire optimization procedures or implicit optima \citep{maclaurin2015,finn2017,rajeswaran2019,lorraine2020}. Process sidecars use the same kind of training-process sensitivity as a deployment-time edit direction: a finite secant estimates the JVP of the realized AdamW safety trace \citep{pearlmutter1994,kingma2015,loshchilov2019}, rather than attributing an example or optimizing a hyperparameter.

\section{Theory: Process Sidecars for Revocable Learned State}
\label{app:theory}

\paragraph{Trainable-coordinate convention.}
All vectors are written in the trainable parameter coordinates used by the experiment. For a LoRA run this means the LoRA adapter coordinates; for a full-finetuning run it means the full parameter vector. This convention avoids irrelevant null coordinates. The norm $\norm{\cdot}$ is Euclidean unless stated otherwise.

\paragraph{Sequential checkpoints.}
The pipeline is
\[
  \theta_A \xrightarrow{\;M\;} \theta_{AM} \xrightarrow{\;S\;} \theta_{AMS},
\]
where $A$ is the skill checkpoint, $M$ is the learned memory phase, and $S$ is the subsequent safety/refusal phase. The counterfactual oracle is
\[
  \theta_{AS} = \Train_S(\theta_A),
\]
which runs the same safety phase from $\theta_A$ rather than $\theta_{AM}$. Let
\[
  u := \Delta_M = \theta_{AM} - \theta_A.
\]
Throughout the appendix, the realized safety-training trace is fixed: the minibatches, random seed, dropout masks if any, optimizer hyperparameters, and training schedule are conditioned on. Thus $T \equiv \Train_S$ is a deterministic map from an initialization to the final safety checkpoint. Randomness can be restored by applying the deterministic statements conditioned on each realized trace.

\subsection{AdamW is an augmented-state differentiable map}
\label{app:adamw}

A central point is that $DT$ is the Jacobian of the entire future training process, not the Hessian of a single loss and not a plain SGD product unless the optimizer really is SGD. For AdamW, the optimizer state must be included.

Let $\ell_t$ be the loss on the $t$-th realized minibatch. We allow optional global-norm clipping with threshold $C \in (0, \infty]$, where $C = \infty$ means no clipping. Define
\[
  \psi_C(g) =
  \begin{cases}
    g, & \norm{g}_2 < C, \\[2pt]
    C g / \norm{g}_2, & \norm{g}_2 > C.
  \end{cases}
\]
The boundary $\norm{g}_2 = C$ is nonsmooth and is excluded by the regular-trace assumption below. We write the parameter at AdamW step $t$ as $\theta(t)$ in this appendix and as $\theta_t$ in the main text; the two are the same sequence and we use them interchangeably. Starting from $\theta(0) = \theta$ and $m_0 = v_0 = 0$, one AdamW trace evolves by
\begin{align}
  g_t &= \nabla \ell_t(\theta(t-1)), & q_t &= \psi_C(g_t), \label{eq:adam-grad} \\
  m_t &= \beta_1 m_{t-1} + (1 - \beta_1) q_t, & v_t &= \beta_2 v_{t-1} + (1 - \beta_2) q_t \odot q_t, \label{eq:adam-mom} \\
  \hat{m}_t &= m_t / (1 - \beta_1^t), & \hat{v}_t &= v_t / (1 - \beta_2^t), \label{eq:adam-bias} \\
  \theta(t) &= (1 - \eta_t \lambda_\wdec)\theta(t-1) - \eta_t\, \hat{m}_t \oslash \left(\sqrt{\hat{v}_t} + \epsilon_\adam\right). && \label{eq:adam-update}
\end{align}
Here $\odot$, $\oslash$, $\sqrt{\cdot}$ are coordinatewise.

\begin{assumption}[Regular AdamW safety trace]
\label{ass:adamw}
There is a radius $r > 0$ around $\theta_A$ such that, for every initialization in $B(\theta_A, r)$:
\begin{enumerate}
  \item each minibatch loss $\ell_t$ is $C^3$ on the induced trace neighborhood;
  \item no clipping decision crosses the boundary $\norm{\nabla \ell_t}_2 = C$;
  \item on the active trainable coordinates, every bias-corrected second moment satisfies $(\hat{v}_t)_i \ge \nu > 0$;
  \item the derivatives of the one-step maps through order three are bounded on this neighborhood.
\end{enumerate}
\end{assumption}

The lower bound on $\hat{v}_t$ is the AdamW square-root regularity condition. PyTorch-style AdamW uses $\sqrt{\hat{v}_t} + \epsilon_\adam$, whose derivative with respect to $\hat{v}_t$ contains $1/\sqrt{\hat{v}_t}$; coordinates with identically zero second moment must be removed from the active set or handled by the smoothed variant $\sqrt{\hat{v}_t + \epsilon_\adam}$.

\begin{theorem}[Exact tangent recursion for multi-step AdamW]
\label{thm:tangent}
Under Assumption~\ref{ass:adamw}, the safety-training map
\[
  T(\theta) = \theta(K)
\]
is differentiable at $\theta_A$. For any initialization perturbation $\delta$, let
\[
  a_t = \frac{\dd \theta(t)}{\dd s}, \qquad
  b_t = \frac{\dd m_t}{\dd s}, \qquad
  c_t = \frac{\dd v_t}{\dd s}
\]
along the perturbed initialization $\theta_A + s\delta$, evaluated at $s = 0$, with $a_0 = \delta$ and $b_0 = c_0 = 0$. Let
\[
  G_t = D(\psi_C \circ \nabla \ell_t)(\theta(t-1)).
\]
When clipping is inactive, $G_t = \nabla^2 \ell_t(\theta(t-1))$. When clipping is active and away from the boundary,
\[
  G_t = \frac{C}{\norm{g_t}_2}\left(I - \frac{g_t g_t^\top}{\norm{g_t}_2^2}\right)\nabla^2 \ell_t(\theta(t-1)).
\]
Then $J_S(\delta) := DT(\theta_A)\delta = a_K$, where
\begin{align}
  h_t &= G_t a_{t-1}, \label{eq:tan-h} \\
  b_t &= \beta_1 b_{t-1} + (1 - \beta_1) h_t, \label{eq:tan-b} \\
  c_t &= \beta_2 c_{t-1} + 2(1 - \beta_2) q_t \odot h_t, \label{eq:tan-c} \\
  \hat{b}_t &= b_t / (1 - \beta_1^t), \qquad \hat{c}_t = c_t / (1 - \beta_2^t), \label{eq:tan-bias} \\
  d_t &= \sqrt{\hat{v}_t} + \epsilon_\adam, \label{eq:tan-d} \\
  a_t &= (1 - \eta_t \lambda_\wdec)a_{t-1} - \eta_t \left( \hat{b}_t \oslash d_t - \tfrac{1}{2}\hat{m}_t \odot \hat{c}_t \oslash \big(\sqrt{\hat{v}_t} \odot d_t^{\odot 2}\big) \right). \label{eq:tan-a}
\end{align}
Equivalently,
\[
  DT(\theta_A) = \Pi_\theta\, DF_K(z_{K-1}) \cdots DF_1(z_0)\, \iota_\theta,
  \qquad
  z_t = (\theta(t), m_t, v_t),
\]
where $\iota_\theta \delta = (\delta, 0, 0)$ and $\Pi_\theta$ projects the final augmented tangent back to parameter coordinates.
\end{theorem}

\begin{proof}
For a fixed trace, one AdamW step is a deterministic map $F_t : (\theta(t-1), m_{t-1}, v_{t-1}) \mapsto (\theta(t), m_t, v_t)$. By Assumption~\ref{ass:adamw}, the composition of the loss gradient, the clipping branch, the moment recurrences, bias correction, coordinatewise square root, and decoupled weight decay is differentiable on the trace neighborhood. Hence $T = \Pi_\theta F_K \circ \cdots \circ F_1 \circ \iota_\theta$ is differentiable and the chain rule gives the augmented-matrix expression.

Expanding the augmented Jacobian product gives the displayed recurrences. The terms for $b_t$ and $c_t$ are the derivatives of the first- and second-moment updates. For the parameter update, differentiate coordinatewise:
\[
  D\!\left(\frac{\hat{m}_i}{\sqrt{\hat{v}_i} + \epsilon_\adam}\right)
  = \frac{\dot{\hat{m}}_i}{\sqrt{\hat{v}_i} + \epsilon_\adam}
  - \frac{\hat{m}_i}{2\sqrt{\hat{v}_i}\,(\sqrt{\hat{v}_i} + \epsilon_\adam)^2}\,\dot{\hat{v}}_i.
\]
Substituting $\dot{\hat{m}} = \hat{b}_t$ and $\dot{\hat{v}} = \hat{c}_t$ into the AdamW parameter update yields the recursion for $a_t$. Therefore $a_K = DT(\theta_A)\delta$.
\end{proof}

\begin{remark}[Why the SGD product is only a corollary]
For actual SGD with decoupled weight decay,
\[
  \theta(t) = (1 - \eta_t \lambda_\wdec)\theta(t-1) - \eta_t \nabla \ell_t(\theta(t-1)),
\]
the tangent is
\[
  DT_{\mathrm{SGD}}(\theta_A) = \prod_{t=K}^{1}\Big( (1 - \eta_t \lambda_\wdec)I - \eta_t \nabla^2 \ell_t(\theta(t-1)) \Big),
\]
with $\nabla^2 \ell_t$ replaced by $G_t$ under clipping. This is not obtained by setting $\beta_1 = \beta_2 = 0$ in AdamW. That limit remains adaptively normalized by $\abs{g_t} + \epsilon_\adam$.
\end{remark}

\subsection{First-order process-sidecar identity}
\label{app:firstorder}

Define the exact transported memory direction and residual sidecar
\[
  p := J_S(u) := DT(\theta_A)u, \qquad r := R_{S\leftarrow M} := p - u.
\]
For the theorem statements in this appendix we work with the exact-derivative edit family
\[
  \theta_{\mathrm{exact}}(\lambda, \gamma) = \theta_{AMS} - \lambda u - \gamma r.
\]
The empirical $\hat{\theta}(\lambda, \gamma)$ used in \S\ref{sec:method} of the main paper replaces $r$ by the centered-secant estimate $\hat{r} = \hat{R}_{S\leftarrow M}$; the secant bias $\norm{\hat{r} - r} = O(\varepsilon^2 \norm{u}^3)$ enters as an additive $O(\varepsilon^2 \norm{u}^3)$ term in the expansions below.

\begin{theorem}[Process-JVP identity and first-order oracle recovery]
\label{thm:jvp-identity}
Suppose $T$ is twice differentiable on $B(\theta_A, r_0)$ and $\norm{D^2 T(x)}_\op \le L_2$ throughout the ball. If $\norm{u} \le r_0$, then
\[
  \theta_{AMS} = \theta_{AS} + p + \rho_2 = \theta_{AS} + u + r + \rho_2,
  \qquad
  \norm{\rho_2} \le \frac{L_2}{2}\norm{u}^2.
\]
Consequently,
\[
  \theta_{\mathrm{exact}}(1, 1) = \theta_{AS} + \rho_2,
  \qquad
  \norm{\theta_{\mathrm{exact}}(1, 1) - \theta_{AS}} \le \frac{L_2}{2}\norm{u}^2.
\]
\end{theorem}

\begin{proof}
Because $\theta_{AM} = \theta_A + u$, we have $\theta_{AMS} = T(\theta_A + u)$ and $\theta_{AS} = T(\theta_A)$. Taylor's theorem with second-order remainder gives
\[
  T(\theta_A + u) = T(\theta_A) + DT(\theta_A)u + \rho_2,
  \qquad
  \norm{\rho_2} \le \frac{L_2}{2}\norm{u}^2.
\]
Substitute the definitions of $p$ and $r = p - u$. The statement for $\theta_{\mathrm{exact}}(1, 1)$ follows by subtracting $u + r = p$.
\end{proof}

\begin{proposition}[Naive task arithmetic misses the transported component]
\label{prop:naive-miss}
Let $P_u$ be Euclidean projection onto $\spanof\{u\}$, and define $r_\perp = (I - P_u)r$. Under Theorem~\ref{thm:jvp-identity}, every naive edit $\theta_{\mathrm{naive}}(\lambda) = \theta_{AMS} - \lambda u$ satisfies
\[
  \norm{\theta_{\mathrm{naive}}(\lambda) - \theta_{AS}} \ge \norm{r_\perp} - \norm{\rho_2}.
\]
The process-sidecar point $\theta_{\mathrm{exact}}(1, 1)$ has error at most $\norm{\rho_2}$. Thus if $\norm{r_\perp} > 2\norm{\rho_2}$, the sidecar point is strictly closer to the counterfactual oracle than every naive edit in parameter norm.
\end{proposition}

\begin{proof}
For the naive line,
\[
  \theta_{\mathrm{naive}}(\lambda) - \theta_{AS} = (1 - \lambda)u + r + \rho_2.
\]
The best choice of $\lambda$ can cancel only the component of $r$ lying in $\spanof\{u\}$; the orthogonal component $r_\perp$ remains. The triangle inequality gives the lower bound. The upper bound for $\theta_{\mathrm{exact}}(1, 1)$ is Theorem~\ref{thm:jvp-identity}.
\end{proof}

\subsection{The centered-secant sidecar used by the implementation}
\label{app:secant}

The code estimates the transported direction by a finite centered secant,
\[
  \hat{J}_{S,\varepsilon}(u) = \frac{T(\theta_A + \varepsilon u) - T(\theta_A - \varepsilon u)}{2\varepsilon}.
\]
For $\varepsilon = 1$, this is a centered secant at the memory-edit scale, not an infinitesimal derivative. The derivative theorem is the local limit.

\begin{proposition}[Centered-secant bias]
\label{prop:secant-bias}
Assume $T \in C^3$ on the segment $\theta_A + su$, $\abs{s} \le \varepsilon$, and $\norm{D^3 T(x)}_\op \le L_3$ there. Then
\[
  \norm{\hat{J}_{S,\varepsilon}(u) - DT(\theta_A)u} \le \frac{L_3}{6}\varepsilon^2 \norm{u}^3.
\]
Consequently,
\[
  \norm{\theta_{AMS} - \hat{J}_{S,\varepsilon}(u) - \theta_{AS}} \le \frac{L_2}{2}\norm{u}^2 + \frac{L_3}{6}\varepsilon^2 \norm{u}^3.
\]
\end{proposition}

\begin{proof}
Apply the one-dimensional Taylor theorem to $\varphi(s) = T(\theta_A + su)$. The centered difference cancels the even second-order term and leaves a third-derivative remainder bounded by $(L_3/6)\varepsilon^2 \norm{u}^3$. Add this estimator bias to the remainder in Theorem~\ref{thm:jvp-identity}.
\end{proof}

\subsection{Second-order frontier: why the two-dimensional family can beat the process-JVP line}
\label{app:frontier}

The first-order theorem does not say that the two-dimensional edit beats the process-JVP line at first order. The process-JVP line $\theta_{AMS} - \eta p$ contains the same first-order oracle at $\eta = 1$. The advantage of the full $(u, r)$ family is second-order and validation-dependent.

Assume in this subsection that $T \in C^3$ near $\theta_A$, and define
\[
  q = \tfrac{1}{2} D^2 T(\theta_A)[u, u].
\]
Then
\[
  \theta_{AMS} = \theta_{AS} + p + q + O(\norm{u}^3).
\]
The residual from $\theta_{AS}$ after a process-JVP edit is
\[
  e_{\mathrm{1D}}(\eta) = q + (1 - \eta)p + O(\norm{u}^3),
\]
whereas the residual after a two-dimensional edit is
\begin{align}
  e_{\mathrm{2D}}(\lambda, \gamma) &= q + (1 - \lambda)u + (1 - \gamma)r + O(\norm{u}^3) \label{eq:e2d-raw} \\
  &= q + \tau p + \zeta r + O(\norm{u}^3), \label{eq:e2d-natural}
\end{align}
with natural coordinates
\[
  \tau = 1 - \lambda, \qquad \zeta = \lambda - \gamma.
\]
The process-JVP line is exactly the subfamily $\zeta = 0$.

\begin{theorem}[Local quadratic frontier expansion]
\label{thm:frontier-app}
Let $\ell$ be a twice-differentiable local scalarization of the validation criterion around $\theta_{AS}$, with expansion
\[
  \ell(\theta_{AS} + e) = \ell_0 + g^\top e + \tfrac{1}{2} e^\top H e + O(\norm{e}^3),
\]
where $H \succ 0$ on $\spanof\{p, r\}$. Suppose the quadratic minimizers below remain in the local neighborhood where the expansion is valid. For any subspace $W \subseteq \spanof\{p, r\}$, define the quadratic surrogate value
\[
  V(W) = \min_{w \in W}\left( g^\top(q + w) + \tfrac{1}{2}(q + w)^\top H (q + w) \right).
\]
Let
\[
  W_1 = \spanof\{p\}, \qquad W_2 = \spanof\{p, r\} = \spanof\{u, r\}.
\]
Then $V(W_2) \le V(W_1)$. Moreover, define
\[
  r_\perp = r - P^H_{\spanof\{p\}} r,
\]
where $P^H$ is $H$-orthogonal projection. If $r_\perp \ne 0$, the improvement is strict iff
\[
  q^\top H r_\perp + g^\top r_\perp \ne 0.
\]
After profiling over the process-JVP coordinate $\tau$, the optimal off-diagonal coordinate is
\[
  \zeta^* = -\frac{q^\top H r_\perp + g^\top r_\perp}{\norm{r_\perp}^2_H}.
\]
\end{theorem}

\begin{proof}
Weak dominance is immediate from $W_1 \subseteq W_2$: the minimum over the larger feasible subspace cannot be worse.

For strictness, first minimize over $W_1$. Let $w^*_1 \in W_1$ be the unique minimizer, which exists because $H \succ 0$ on the plane. The first-order condition says
\[
  g^\top h + \inner{q + w^*_1}{h}_H = 0 \quad \text{for all } h \in W_1.
\]
Since $r_\perp$ is $H$-orthogonal to $W_1$, varying from the profiled one-dimensional optimum in the off-diagonal direction gives the one-variable quadratic
\[
  \phi(\zeta) = g^\top(q + w^*_1 + \zeta r_\perp) + \tfrac{1}{2}\norm{q + w^*_1 + \zeta r_\perp}^2_H.
\]
Its derivative at zero is
\[
  \phi'(0) = g^\top r_\perp + \inner{q + w^*_1}{r_\perp}_H = g^\top r_\perp + q^\top H r_\perp,
\]
because $w^*_1 \in W_1$ and $r_\perp \perp_H W_1$. Its curvature is $\phi''(\zeta) = \norm{r_\perp}^2_H > 0$. Hence adding the off-diagonal direction strictly improves the value exactly when $\phi'(0) \ne 0$, and minimizing $\phi$ gives the displayed $\zeta^*$.
\end{proof}

\begin{remark}[What the theorem does and does not say]
The theorem proves weak dominance of the two-dimensional local quadratic frontier over the process-JVP line, and strict dominance under a Hessian-weighted off-diagonal projection condition. It does not prove that Hessian anisotropy alone is sufficient. Anisotropy is neither necessary nor sufficient: if the profiled off-diagonal derivative $q^\top H r_\perp + g^\top r_\perp$ vanishes, the process-JVP line is already locally optimal even under an anisotropic $H$; conversely, with $H = I$, any quadratic correction with a nonzero component along $r_\perp$ gives strict two-dimensional improvement.
\end{remark}

\begin{corollary}[Constrained-validation/KKT interpretation]
\label{cor:kkt}
Let the validation selector minimize $f_0(e) = -M(e)$ subject to smooth local constraints $c_j(e) \le 0$, e.g.\ secret-AUC and skill-retention thresholds. At a regular local constrained optimum, let $\mu_j \ge 0$ be active KKT multipliers and define the Lagrangian scalarization
\[
  \mathcal{L}(e, \mu) = f_0(e) + \sum_j \mu_j c_j(e).
\]
If the Hessian $H_\mathcal{L} = \nabla^2_e \mathcal{L}(0, \mu)$ satisfies the second-order sufficient condition on the candidate plane or on the relevant critical cone, then Theorem~\ref{thm:frontier-app} applies with $g = \nabla_e \mathcal{L}(0, \mu)$ and $H = H_\mathcal{L}$. Thus active validation constraints change the scalarization and Hessian, but not the geometric criterion: strict off-diagonal improvement is governed by the $H_\mathcal{L}$-weighted projection of the second-order/KKT correction onto $r_\perp$.
\end{corollary}

\begin{remark}[Indefinite Hessians]
If $H$ is indefinite, the unconstrained projection formula is not a valid global minimization statement. The safe alternative is a trust-region quadratic comparison,
\[
  \min_{w \in W,\, \norm{w} \le \rho} g^\top(q + w) + \tfrac{1}{2}(q + w)^\top H (q + w).
\]
Weak dominance still follows from $W_1 \cap B_\rho \subseteq W_2 \cap B_\rho$. The explicit projection and $\zeta^*$ formula should be used only when the local second-order sufficient condition makes the relevant quadratic positive definite on the searched plane, or when the trust-region optimum is interior and the restricted Hessian is positive definite there.
\end{remark}

\subsection{Oracle-free finite-grid selector}
\label{app:selector}

Let $\mathcal{G}$ be the declared finite edit grid. For each candidate $a \in \mathcal{G}$, define population validation quantities
\[
  A(a) = \text{secret AUC}, \quad
  R(a) = \text{skill retention}, \quad
  M(a) = \text{refusal margin}, \quad
  N(a) = \text{edit norm}.
\]
The population decision rule is
\[
  \max_{a \in \mathcal{G}} M(a) \quad \text{subject to} \quad A(a) \le a_0, \quad R(a) \ge r_0,
\]
with minimum $N(a)$ as a deterministic tie-break. The empirical selector uses validation estimates $\hat{A}, \hat{R}, \hat{M}$ computed before any test evaluation. In the experiments, $a_0 = 0.6$ and $r_0 = 0.9$.

\begin{theorem}[Slacked finite-grid selector guarantee]
\label{thm:selector}
Suppose that, with probability at least $1 - \delta_c$, all candidates satisfy
\[
  \abs{\hat{A}(a) - A(a)} \le \epsilon_A(a), \qquad \abs{\hat{R}(a) - R(a)} \le \epsilon_R(a).
\]
Let the empirical selector use tightened feasibility
\[
  \hat{A}(a) \le a_0 - \epsilon_A(a), \qquad \hat{R}(a) \ge r_0 + \epsilon_R(a),
\]
and choose among feasible candidates the one maximizing $\hat{M}$, with minimum-norm tie-break. Then the selected candidate $\hat{a}$ is population feasible: $A(\hat{a}) \le a_0$ and $R(\hat{a}) \ge r_0$.

Furthermore, suppose that, with probability at least $1 - \delta_m$, for every ordered pair $(a, b) \in \mathcal{G}^2$,
\[
  M(a) - M(b) \ge \hat{M}(a) - \hat{M}(b) - \mathrm{rad}_{ab}.
\]
Then, on the joint event, for every candidate $b$ satisfying the stronger population slack
\[
  A(b) \le a_0 - 2\epsilon_A(b), \qquad R(b) \ge r_0 + 2\epsilon_R(b),
\]
we have
\[
  M(\hat{a}) \ge M(b) - \mathrm{rad}_{\hat{a}b}.
\]
In particular, $\hat{a}$ is near-optimal over the slacked feasible grid, with the loss controlled by paired candidate-difference uncertainty rather than by a raw uniform Hoeffding radius for $M$.
\end{theorem}

\begin{proof}
The feasibility claim follows immediately: if $\hat{A}(\hat{a}) \le a_0 - \epsilon_A(\hat{a})$ and $\abs{\hat{A} - A} \le \epsilon_A$, then $A(\hat{a}) \le a_0$; the retention constraint is identical with the inequality reversed.

If $b$ satisfies the stronger population slack, then the same concentration event implies $\hat{A}(b) \le a_0 - \epsilon_A(b)$ and $\hat{R}(b) \ge r_0 + \epsilon_R(b)$, so $b$ is empirically feasible. Since $\hat{a}$ maximizes $\hat{M}$ over the empirical feasible set, $\hat{M}(\hat{a}) - \hat{M}(b) \ge 0$, up to deterministic tie-breaking. The pairwise margin event gives
\[
  M(\hat{a}) - M(b) \ge \hat{M}(\hat{a}) - \hat{M}(b) - \mathrm{rad}_{\hat{a}b} \ge -\mathrm{rad}_{\hat{a}b}. \qedhere
\]
\end{proof}

\begin{proposition}[Paired empirical-Bernstein radius]
\label{prop:bernstein}
Let $M(a) = \E[m_i(a)]$, where validation units are independent and shared across candidates. For an ordered pair $(a, b)$, define
\[
  X^{ab}_i = m_i(a) - m_i(b), \qquad
  \hat{D}_{ab} = \frac{1}{n}\sum_i X^{ab}_i, \qquad
  \hat{\sigma}^2_{ab} = \frac{1}{n - 1}\sum_i (X^{ab}_i - \hat{D}_{ab})^2.
\]
Assume the random variables $X^{ab}_i$ lie almost surely in an interval of length $B_{ab}$. A standard empirical-Bernstein union bound gives, simultaneously over all ordered pairs, with probability at least $1 - \delta$,
\[
  M(a) - M(b) \ge \hat{D}_{ab} - \sqrt{\frac{2\hat{\sigma}^2_{ab}\log(2\abs{\mathcal{G}}^2/\delta)}{n}} - \frac{7 B_{ab}\log(2\abs{\mathcal{G}}^2/\delta)}{3(n - 1)}.
\]
Thus Theorem~\ref{thm:selector} can use this data-dependent $\mathrm{rad}_{ab}$. When candidates are evaluated on the same validation prompts, the paired variance $\hat{\sigma}^2_{ab}$ can be far smaller than the marginal variance of either candidate's raw margin.
\end{proposition}

\begin{remark}[What this selector theorem certifies]
This is an oracle-free finite-grid theorem. It does not claim global optimality over all possible unlearning methods. It says that the fixed validation rule solves the declared constrained deployment problem over $\mathcal{G}$, up to explicit feasibility and paired-difference uncertainty. The grid may contain naive task arithmetic ($\gamma = 0$), the process-JVP line ($\gamma = \lambda$), and the full two-dimensional family. A nonzero off-diagonal coefficient wins only if it wins the same predeclared validation problem.
\end{remark}

\begin{proposition}[Across-trial sign-test validation]
\label{prop:signtest}
For held-out comparisons across independent trials, let $D_j$ be the test metric difference between two fixed procedures on trial $j$. Under the null hypothesis that the median difference is nonpositive, and assuming no ties or random tie-breaking, observing $s$ positive signs in $T$ trials has one-sided $p$-value
\[
  \sum_{k=s}^{T} \binom{T}{k} 2^{-T}.
\]
This validates a fixed method comparison, not the finite-sample optimality of a within-trial selector. It is nevertheless the correct distribution-free test for claims such as ``the selected two-dimensional method beats a matched process-JVP baseline across independent trials.''
\end{proposition}

\subsection{Gradient-ascent unlearning: conditional safety drift, not impossibility}
\label{app:drift}

Let an FT-unlearning procedure update
\[
  \theta_{t+1} = \theta_t + \alpha d_t,
\]
where $d_t$ is the actual ascent/update direction produced by the procedure for the memory objective, possibly including optimizer preconditioning and retention terms. Let $L_S$ be a safety loss.

\begin{proposition}[Safety drift under positive-curvature memory ascent]
\label{prop:safety-drift}
Assume $L_S \in C^3$ along the unlearning trajectory and that its Hessian is $L_3$-Lipschitz there. Then each step satisfies
\[
  L_S(\theta_{t+1}) - L_S(\theta_t) = \alpha \inner{\nabla L_S(\theta_t)}{d_t} + \frac{\alpha^2}{2} d_t^\top \nabla^2 L_S(\theta_t) d_t + R_t,
\]
with
\[
  \abs{R_t} \le \frac{L_3}{6}\alpha^3 \norm{d_t}^3.
\]
If
\[
  \abs{\inner{\nabla L_S(\theta_t)}{d_t}} \le \epsilon_t \norm{d_t}, \qquad
  d_t^\top \nabla^2 L_S(\theta_t) d_t \ge \Lambda_t \norm{d_t}^2,
\]
then after $K$ steps,
\[
  L_S(\theta_K) - L_S(\theta_0) \ge \sum_{t=0}^{K-1}\left( -\alpha \epsilon_t \norm{d_t} + \frac{\alpha^2}{2}\Lambda_t \norm{d_t}^2 - \frac{L_3}{6}\alpha^3 \norm{d_t}^3 \right).
\]
If additionally $\norm{d_t} \in [G_{\min}, G_{\max}]$, $\epsilon_t \le \epsilon$, and $\Lambda_t \ge \Lambda > 0$, then
\[
  L_S(\theta_K) - L_S(\theta_0) \ge K\left( -\alpha \epsilon G_{\max} + \frac{\alpha^2}{2}\Lambda G_{\min}^2 - \frac{L_3}{6}\alpha^3 G_{\max}^3 \right).
\]
\end{proposition}

\begin{proof}
Taylor expand $L_S(\theta_t + \alpha d_t)$ to second order with third-order Lipschitz-Hessian remainder. Lower-bound the linear term by $-\alpha \epsilon_t \norm{d_t}$, lower-bound the quadratic term by $(\alpha^2/2)\Lambda_t \norm{d_t}^2$, upper-bound the magnitude of the remainder, and sum over $t$.
\end{proof}

\begin{corollary}[Eigenprojection form]
\label{cor:eigen}
Let $P_k$ be the orthogonal projector onto an invariant subspace of $\nabla^2 L_S(\theta_t)$, equivalently a span of its eigenvectors, on which $\nabla^2 L_S(\theta_t) \succeq \lambda_k I$. Suppose the invariant orthogonal complement has curvature lower bounded by $-\beta I$. If $\norm{P_k d_t}^2 \ge \rho_t \norm{d_t}^2$, then
\[
  d_t^\top \nabla^2 L_S(\theta_t) d_t \ge \big( \rho_t \lambda_k - (1 - \rho_t)\beta \big)\norm{d_t}^2.
\]
Thus the proposition applies with $\Lambda_t = \rho_t \lambda_k - (1 - \rho_t)\beta$ whenever this quantity is positive.
\end{corollary}

\begin{remark}[Scope]
This is a local failure-mode theorem. It does not prove that FT-unlearning is impossible, and it should not be used to predict the magnitude of large non-perturbative refusal collapse. It says that once safety training has made the first-order safety gradient small, memory-ascent directions with positive projection onto safety-curvature directions increase safety loss to second order.
\end{remark}

\subsection{Boundary cases for the second-order frontier}
\label{app:boundary}

\paragraph{Hessian anisotropy alone is insufficient.}
Let $p = (1, 1)$, $r = (0, 1)$, $H = \diag(1, 10)$, $g = 0$, and $q = 2p$. The Hessian is anisotropic on the plane, but $q \in \spanof\{p\}$, so the process-JVP line cancels the second-order residual exactly. The full two-dimensional family cannot improve on zero residual.

\paragraph{Hessian anisotropy is not necessary.}
Let $H = I$, $p = (1, 0)$, $r = (0, 1)$, $g = 0$, and $q = r$. The process-JVP line cannot cancel the $r$ component, whereas the two-dimensional family can. Strict improvement holds with an isotropic Hessian.

\paragraph{Raw $\gamma$ sign is not invariant.}
The natural off-diagonal coordinate is $\zeta = \lambda - \gamma$. Reparameterizing the same edit family changes the raw coefficient attached to $r$, but the condition in Theorem~\ref{thm:frontier-app} predicts the sign of $\zeta^*$, not a universal sign for $\gamma$.

\paragraph{Finite-scale secant.}
The centered-secant theorem is asymptotic. If the implemented estimator is
\[
  \hat{J}_{S,\varepsilon}(u) = \frac{T(\theta_A + \varepsilon u) - T(\theta_A - \varepsilon u)}{2\varepsilon},
\]
then using it in the sidecar point gives
\[
  \theta_{AMS} - \hat{J}_{S,\varepsilon}(u) - \theta_{AS} = O(\norm{u}^2) + O(\varepsilon^2 \norm{u}^3).
\]
Thus $\varepsilon = 1$ is a centered secant at the memory-edit scale. The derivative theorem is the local limiting statement; finite-scale accuracy is an empirical diagnostic.

\paragraph{Indefinite projected Hessians.}
The closed-form coefficient
\[
  \zeta^* = -\frac{q^\top H r_\perp + g^\top r_\perp}{\norm{r_\perp}^2_H}
\]
is a minimizer only under the positive-definite or second-order-sufficient condition stated in Theorem~\ref{thm:frontier-app}. If the projected Hessian on $\spanof\{p, r\}$ is indefinite, this expression may describe a stationary point or may be undefined as an optimization certificate. In that case we use the trust-region nesting statement
\[
  W_1 \cap B_\rho \subseteq W_2 \cap B_\rho
\]
and the empirical grid selector. Any magnitude correlation involving $\zeta^*$ should be reported only on trials satisfying the theorem's curvature hypothesis, or explicitly labeled as a diagnostic score rather than a theorem-predicted minimizer.

\begin{lemma}[Linear-dominated off-diagonal sign diagnostic]
\label{lem:sign}
Let
\[
  \ell(e) = \ell_0 + g^\top e + \tfrac{1}{2} e^\top H e + O(\norm{e}^3)
\]
be the local scalar objective to minimize, and write $e(\tau, \zeta) = q + \tau p + \zeta r$. Let $r_E = r - P^E_{\spanof\{p\}} r$ be the Euclidean component of $r$ not reachable by the process-JVP line. After profiling over $\tau$, suppose the objective variation in the $\zeta$ direction is linear-dominated on the searched interval: for all searched $\zeta$,
\[
  \abs{\inner{H e(\tau(\zeta), \zeta)}{r_E}} + O(\norm{e}^2 \norm{r_E}) < \abs{\inner{g}{r_E}}.
\]
Then any local decrease from the process-JVP line moves in the direction
\[
  \sign(\zeta) = -\sign\inner{g}{r_E}.
\]
This sign diagnostic does not require $H$ to be positive definite, but it also does not give the positive-definite magnitude formula for $\zeta^*$.
\end{lemma}

\begin{proof}
The directional derivative of the profiled objective in the off-diagonal direction is
\[
  \partial_\zeta \ell(e(\tau(\zeta), \zeta)) = \inner{g + H e(\tau(\zeta), \zeta)}{r_E} + O(\norm{e}^2 \norm{r_E}),
\]
where the profiled $\tau$ term drops out along the component orthogonal to the line. Under the stated dominance condition, this derivative has the sign of $\inner{g}{r_E}$. Moving in the opposite sign of $\zeta$ decreases the objective.
\end{proof}

\paragraph{Selector certificates.}
The paired empirical-Bernstein theorem formalizes the finite-grid selector, but it is only an empirical certificate when the resulting radius is smaller than the observed paired margin gap. If the computed radius is larger than the effect at the validation size used here, the primary statistical evidence is the across-trial sign test over independent trials.

\paragraph{FT-unlearn magnitude.}
The safety-drift proposition is local and should not be used to predict the magnitude of large refusal-collapse values. The correct claim is directional: small first-order safety gradient plus positive safety curvature of the memory ascent direction produces second-order safety loss. Observed large collapses are empirical outcomes outside the perturbative magnitude regime.

\section{Experimental details}
\label{app:expdetails}

\paragraph{Data and phases.}
Each synthetic trial constructs a public skill split $D_A$, a memory split $D_M$, and a refusal split $D_S$. The memory set contains 64 high-entropy canaries paired with entity identifiers; the refusal set includes entity-specific prompts and decoys. We train $A$, then $M$ from $A$, then $S$ from $AM$. The oracle $AS$ runs the same $S$ phase from $A$.

\paragraph{Candidate grids.}
The 2D process-sidecar grid uses $\lambda \in \{0.75, 1.0, 1.25, 1.5\}$ and $\gamma \in \{-0.5, 0, 0.5, 1.0, 1.5, 2.0\}$ unless otherwise stated. Naive candidates are $\gamma = 0$ with the same $\lambda$ values. Process-1D candidates are $\gamma = \lambda$. The gamma-extension audit at Qwen-2.5-0.5B-Instruct extends $\gamma$ through $5.0$ and selects an interior point $(1.5, 2.5)$ rather than a maximum-grid corner.

\paragraph{Selector.}
The selector filters on validation secret AUC $\le 0.60$, validation skill retention $\ge 0.90$, and validation refusal preference rate $\ge 0.99$. It maximizes the validation refusal-margin improvement over $\theta_{AMS}$, transformed by $\log(1 + \max\{0, \cdot\})$, and breaks exact ties by edit norm. The selector is run separately for each declared family.

\paragraph{Statistics.}
The main intervals are percentile bootstrap intervals over independent trials. The matched-candidate claim reports paired per-trial differences and, for the primary statistical certificate, a sign test over 15 model-by-data-seed blocks that allows arbitrary dependence within each four-trial block. We do not use cluster-bootstrap confidence intervals because the number of model-level clusters is too small for a reliable 95\% bootstrap interval.

\paragraph{Compute, artifacts, and asset licensing.}
The reported grids ran on single-GPU workers: RTX 6000 Ada-class workers for the Qwen-2.5-0.5B-Instruct runs, A100 80GB-class workers for the Qwen-2.5-1.5B-Instruct runs and heavy baseline sweeps, and RTX 4090/A100-class workers for the Llama-3.2-1B-Instruct and 9B scale-probe runs. Single trials took roughly 10 minutes at 0.5B and 30--60 minutes at 1--1.5B depending on baseline coverage; the reported main, baseline, HVP, and fresh-confirmatory experiments used roughly \$60--\$80 of cloud GPU time. The supplement contains the training/editing code paths, selector scripts, analysis scripts, synthetic data generators, cached summary JSON files, figure-generation scripts, training-loss trajectories per phase, and a Python dependency manifest. We release synthetic canary generators and cached summaries, not trained checkpoints containing memorized canaries. Qwen-2.5-0.5B-Instruct/1.5B-Instruct, Qwen3-8B, and Qwen3.5-9B are used under Apache-2.0 \citep{qwen2024_05b,qwen2024_15b,qwen2025_8b_card,qwen2026_9b_card}, Llama-3.2-1B-Instruct under the Llama 3.2 Community License \citep{metaai2024}, and the WMDP/RMU baseline code under the MIT license \citep{cais2024}; NPO is implemented from the cited paper rather than vendored as an external package \citep{zhang2024npo}.

\section{HVP probe audit details}
\label{app:hvp}

For eight cached Qwen-2.5-1.5B-Instruct trials, we estimated a reduced Hessian at $\theta_{AS}$ by centered finite differences of the scalar safety validation loss. The theory is stated in raw $p = J_S(u)$ and $r = p - u$ coordinates, whereas this cache uses unit-normalized probe directions, and the profiled-$\tau$ linear-dominance hypothesis of Lemma~\ref{lem:sign} is not verified at grid scale. We therefore use this appendix as a curvature-hypothesis audit: the projected Hessian is positive definite in $3/8$ cached probes.

\begin{table}[h]
  \caption{Reduced-Hessian curvature audit on cached 1.5B trials. PD means the two-dimensional projected Hessian is positive definite. The selector evidence in the main text comes from held-out validation and matched-family test comparisons.}
  \label{tab:hvp}
  \centering
  \begin{tabular}{ccc}
    \toprule
    Trial & PD & $\lambda_{\min}$ \\
    \midrule
    0  & no  & $-1.04 \times 10^{-2}$ \\
    1  & no  & $-8.21 \times 10^{-3}$ \\
    2  & no  & $-1.53 \times 10^{-3}$ \\
    3  & yes & $+1.02 \times 10^{-2}$ \\
    4  & no  & $-8.73 \times 10^{-3}$ \\
    5  & no  & $-1.23 \times 10^{-2}$ \\
    10 & yes & $+3.81 \times 10^{-3}$ \\
    15 & yes & $+7.16 \times 10^{-3}$ \\
    \bottomrule
  \end{tabular}
\end{table}

\section{Matched-candidate selector details}
\label{app:matched}

The matched-candidate analysis reuses the cached 24-cell 2D grid. The diagonal subset $\gamma = \lambda$ is the process-JVP line inside that grid; the $\gamma = 0$ subset is naive task arithmetic. Running the same selector on these subsets controls evaluation noise and family restriction but not candidate count: the diagonal slice contains 2 cells and the $\gamma = 0$ slice contains 4, against the 24-cell 2D grid. Candidate count is controlled separately by the equal-density 24-cell dense process-JVP diagonal in \S\ref{sec:offdiag}.

\begin{table}[h]
  \caption{Per-run matched-candidate results.}
  \label{tab:matched}
  \centering
  \begin{tabular}{lcccc}
    \toprule
    Run & $n$ & 2D--matched-1D mean & 95\% CI & sign \\
    \midrule
    Qwen-2.5-0.5B-Instruct & 20 & $+0.202$ & $[+0.164, +0.241]$ & $20/20$ \\
    Qwen-2.5-1.5B-Instruct & 20 & $+0.364$ & $[+0.277, +0.470]$ & $20/20$ \\
    Llama-3.2-1B-Instruct  & 20 & $+0.381$ & $[+0.264, +0.513]$ & $20/20$ \\
    \bottomrule
  \end{tabular}
\end{table}

\section{Baseline hyperparameters and configurations}
\label{app:baselines}

All baselines share the LoRA setup of the main edits: rank 8, alpha 16, target modules $\{q, k, v, o\}_{\mathrm{proj}}$ and $\{\mathrm{up}, \mathrm{down}, \mathrm{gate}\}_{\mathrm{proj}}$ for the transformer scales (Qwen-2.5-0.5B-Instruct/1.5B-Instruct and Llama-3.2-1B-Instruct); the 9B run additionally targets the SSM-block projections $\{\mathrm{in}^{qkv}_{\mathrm{proj}}, \mathrm{in}^{z}_{\mathrm{proj}}, \mathrm{in}^{b}_{\mathrm{proj}}, \mathrm{in}^{a}_{\mathrm{proj}}, \mathrm{out}_{\mathrm{proj}}\}$. The same retain dataset (skill prompts and safety-general prompts) and the same memory-canary forget dataset are used for all baselines.

\paragraph{FT-unlearn.}
Let $L_{\mathrm{forget}}, L_{\mathrm{skill}}, L_{\mathrm{safety}}$ denote the negative log-likelihoods on the memory-canary, public-skill, and safety-general data respectively. A single optimization step minimises
\[
  L_{\mathrm{FT}} = -\alpha_{\mathrm{forget}} L_{\mathrm{forget}} + \alpha_{\mathrm{skill}} L_{\mathrm{skill}} + \alpha_{\mathrm{safety}} L_{\mathrm{safety}},
\]
so the negative sign on $L_{\mathrm{forget}}$ implements gradient ascent on the memory-canary NLL (i.e.\ unlearning) while the retain terms preserve skills and safety. Default learning rate $10^{-4}$, AdamW, batch size 8, one epoch. The FT-unlearn grid sweeps the six configurations
\[
  (\alpha_{\mathrm{forget}}, \alpha_{\mathrm{skill}}, \alpha_{\mathrm{safety}}) \in \{(0.5, 0.3, 0.5), (1.0, 0.3, 0.5), (2.0, 0.3, 0.5), (1.0, 1.0, 0.5), (1.0, 0.3, 1.0), (1.0, 1.0, 2.0)\}.
\]
The default-trial baseline sweep uses $(1.0, 0.3, 0.5)$.

\paragraph{NPO.}
Negative Preference Optimization \citep{zhang2024npo}: $\beta = 0.1$, NPO loss coefficient 1.0, gradient-difference retain coefficient 1.0, learning rate $10^{-4}$, AdamW, batch size 8, one epoch. Retain set as above.

\paragraph{RMU.}
Base-parameter RMU baseline, following the WMDP-RMU objective \citep{li2024wmdp}: layer ids $\{3, 4, 5\}$, parameter id 6 (the MLP block of the steering layer), steering coefficient 6.5, forget-loss weight $\alpha = 1200$, learning rate $5 \times 10^{-5}$, AdamW. This row updates selected base-model parameters rather than LoRA adapter coordinates, so behavioral metrics are reported but LoRA-coordinate weight-space closure is not.

\paragraph{TIES, DARE.}
TIES is per-tensor magnitude trim from mergekit \citep{goddard2024}; DARE is the canonical per-element Bernoulli$(d)$ drop with $1/d$ rescale from \citet{yu2024}. We use the two-line DARE formula directly because the mergekit Bernoulli helper does not apply this exact rescale. We sweep density $d \in \{0.2, 0.5, 0.8\}$ and scale $\lambda \in \{0.75, 1.0, 1.25, 1.5\}$ (twelve cells per method per trial), then apply the same selector. Identical retain and validation pipelines to the main 2D experiments.

\paragraph{Result summary at Qwen-2.5-1.5B-Instruct.}
The RMU/NPO/FT baseline sweep at the 1.5B scale has five trials. FT-unlearn has mean refusal closure $-43.39$ and median $-47.08$. NPO has mean refusal closure $-1.33$. RMU has refusal closure approximately zero and secret AUC remains high rather than reaching chance. The structured-synthetic canary has five trials; process-2D wins $5/5$ against naive on refusal, and FT-unlearn has mean refusal closure $-146.75$.

\section{Broader impacts}
\label{app:impacts}

Process sidecars are intended for privacy- and compliance-driven deletion when a model has learned private or proprietary facts before later safety training. The intended deployment is specific: remove a revocable memory while preserving the refusal behavior that protects the affected entities. The same mechanism would be harmful if used to strip safety policies or to create false assurance about deletion, so deployment should require held-out forgetting tests, refusal-preservation tests, and audit logs tying an edit to the realized training trace.

\section{Reproducibility statement}
\label{app:repro}

The experiments use fixed seeds and cached per-trial JSON summaries; trained state tensors are omitted from the supplement because they can contain memorized canaries. The selector is specified before test evaluation and does not use the oracle checkpoint. The appendix gives the exact theorem assumptions, the finite-grid selector statement, and the matched-candidate and HVP diagnostics.

\end{document}